\documentclass[12pt,a4paper]{article}
\usepackage{graphicx}
\usepackage{amsmath,amssymb}
\usepackage{color}
\usepackage[affil-it]{authblk} %for affiation
\usepackage{dsfont} %for indicator function 1
\usepackage{subcaption}
\usepackage{times}
\usepackage{amsthm}
\usepackage{array, tabu}
\usepackage{makecell,interfaces-makecell,siunitx,booktabs,caption}

\usepackage[french,english]{babel}
\usepackage[utf8]{inputenc}
\usepackage[french]{babel}
\theoremstyle{definition}

\newtheorem{theorem}{Theorem}
\usepackage{float}

\usepackage{footnote}
\usepackage[pagebackref=true,breaklinks=true,letterpaper=true,colorlinks,bookmarks=false]{hyperref}

\def\x{\mathbf{x}}

\def\z{\mathbf{z}}
\def\1{\mathbf{1}}
\def\0{\mathbf{0}}
\def\P{\mathbf{P}}

\def\optlimits{\nolimits}

\begin{document}
\title{An ILP Solver for Multi-label MRFs with Connectivity Constraints}

\author[1]{Ruobing Shen%
  \thanks{Electronic address: \texttt{ruobing.shen@informatik.uni-heidelberg.de}}}
\affil[1]{Institue of Computer Science, Heidelberg University, Germany}

\author{Eric Kendinibilir}
\affil{Institute of Mathematics, Heidelberg University, Germany}
\author{Ismail Ben Ayed}
\affil{LIVIA Laboratory, École de technologie supérieure (ETS), Montreal, QC, Canada}

\author{Andrea Lodi%
  %\thanks{Electronic address: \texttt{andrea.lodi@unibo.it}}
  }
\affil{Department of Mathematical and Industrial Engineering, Polytechnique Montréal, Canada}
\author{Andrea Tramontani%
  %\thanks{Electronic address: \texttt{andrea.tramontani@it.ibm.com}}
  }
\affil{CPLEX Optimization, IBM, Italy}
\author[1]{Gerhard Reinelt%
  }

% For a paper whose authors are all at the same institution,
% omit the following lines up until the closing ``}''.
% Additional authors and addresses can be added with ``\and'',
% just like the second author.
% To save space, use either the email address or home page, not both
%\and
%Gerhard Reinelt\\
%Institute of Computer Science, Heidelberg University\\
%Im Neuenheimer Feld 205, 69120 Heidelberg, Germany\\
%{\tt\small gerhard.reinelt@informatik.uni-heidelberg.de}

\maketitle

\begin{abstract}
%%%%%%%%%%%%%%%%%%%%%%%%%%%%%%%% Abstract quite long. This is my refinement. 
Integer Linear Programming (ILP) formulations of Markov random fields (MRFs) models with global connectivity priors were investigated previously in computer vision, e.g., \cite{globalinter,globalconn}. In these works, only Linear Programing (LP) relaxations \cite{globalinter,globalconn} or simplified versions \cite{graphcutbase} of the problem were solved. This paper investigates the ILP of multi-label MRF with exact connectivity priors via a branch-and-cut method, which provably finds globally optimal solutions. The method enforces connectivity priors iteratively by a cutting plane method, and provides feasible solutions with a guarantee on sub-optimality even if we terminate it earlier. The proposed ILP can be applied as a post-processing method on top of any existing multi-label segmentation approach. As it provides globally optimal solution, it can be used off-line to generate ground-truth labeling, which serves as quality check for any fast on-line algorithm. Furthermore, it can be used to generate ground-truth proposals for weakly supervised segmentation. 

We demonstrate the power and usefulness of our model by several experiments on the BSDS500 and PASCAL image dataset, as well as on medical images with trained probability maps. 
%, thus enables us to perform a quantitative comparison between exact and approximative algorithms. 
\end{abstract}

\section{Introduction}
\label{sec:intro}

Most early vision problems can be formulated using Markov Random Fields (MRFs), hence its solution algorithms are of pivotal importance in computer vision. The MAP-MRF (maximizing a posteriori in an MRF) has proven to be successful for many computer vision problems such as image segmentation, denoising and stereo, among others. We refer to \cite{mrf-d,mrf,mrf-s} for an overview of MRF optimization techniques and applications in vision.

In the standard case of MRF with pairwise potentials, we have an undirected graph $G=(V,E)$, where~$V$ represents a set of pixels (or superpixels) from an input image, and~$E$ denotes a set of edges consisting of unordered pairs of nodes indicating adjacency relations. We consider the problem of 
minimizing the following energy function:
\begin{equation}
E(x)=\sum\optlimits_{p\in V}\theta_p(x_p)+\sum\optlimits_{(p,q)\in E}\theta_{pq}(x_p,x_q).
\label{energy}
\end{equation}

Here, we use $x_p$ to denote the label of node $p\in V$, which belongs to a pre-defined finite set $\mathcal{L}=[k]$ representing $k$ classes, where $[k]=\{1,\ldots,k\}$. $\theta_p(x_p)$ is usually called unary potential, and is derived from the observed data. It measures how well label~$x_p$ fits node~$p$. $V_{pq}(x_p,x_q)$ is often referred to as pairwise potential. It measures the cost of assigning labels $x_p,x_q$ to adjacent nodes $p,q$. Typically, it is used to impose spatial smoothness or to align the solution boundaries to image edges.
The goal is to find a labeling~$\x$ (i.e., a mapping from~$V$ to $\mathcal{L}$) that minimizes~$E(x)$. The Potts function~$\theta(\alpha,\beta)=\lambda\cdot \mathds{1}(\alpha\neq\beta)$, where $\lambda$ is a constant, and $\mathds{1}(\cdot)$ is $1$ if its argument is true and $0$ otherwise, is widely used, among many other functionals.

%At object boundaries, adjacent nodes should often have very different labels and it is crucial that $E$ does not over-penalize such labellings. This requires that $V$ be a nonconvex function of $|x_p-x_q|$, and is called \emph{discontinuity-preserving}. 

Minimizing energy \eqref{energy} is a difficult problem ($\mathcal{NP}$-hard in general). In the case of an undirected graph, and by introducing binary variables $x_i^\ell$, $i\in V$, $\ell\in \mathcal{L}$, which indicate whether node $i$ is assigned label $\ell$ ($x_i^\ell=1$ in this case), the corresponding ILP formulation with Potts function boils down to:
%There exists efficient approximation algorithms, like \emph{$\alpha$-expansion} \cite{graphcut}, 

\begin{alignat}{2}
\text{min}_{x}\;\;(1-\lambda)\sum_{\ell=1}^k&\sum_{i=1}^n c_i^\ell x_{i}^\ell +\lambda\sum_{\ell=1}^k\sum_{(i,j)\in E}|x_i^\ell-x_j^\ell|\label{ILP}\\
\sum\optlimits_{\ell=1}^kx_{i}^\ell&=1, \;\;\forall i\in [n], \label{ILP1}\tag{\ref{ILP}a} \\
x_{i}^\ell&\in\{0,1\}, \;\;\forall i\in [n], \;\; \ell\in [k], \tag{\ref{ILP}b}
\end{alignat}
where $c_i^\ell$ denotes the unary data term for label~$\ell$ and node~$i$, and~$\lambda\in[0,1]$ is a positive parameter weighting the contribution of the smoothness term. 
Constraint~\eqref{ILP1} enforces that each node is assigned exactly one label.

Since~\eqref{ILP} is $\mathcal{NP}$-hard and difficult to solve to optimality, it is common in vision to solve the corresponding LP relaxation~\cite{pd2,fastpd}. There have been works on solving approximations of~\eqref{ILP}, for instance, message passing algorithms~\cite{tree,bp} and $\alpha$-expansion~\cite{graphcut} with guaranteed approximation ratios. The corresponding condition for $\alpha$-expansion is nonnegative edge weights and $V_{pq}(\beta,\gamma)+V_{pq}(\alpha,\alpha)\leq V_{pq}(\beta,\alpha)+V_{pq}(\alpha,\gamma)$, for all labels $\alpha, \beta, \gamma\in \mathcal{L}$. On the other hand, it is also important to solve~\eqref{ILP} to optimality (even off-line), thus providing ground-truth benchmarks for those fast approximate algorithms.

The standard model in \eqref{ILP}, which combines unary and pairwise potentials, can impose only a limited class of constraints on the solution. Therefore, there is an ongoing research effort in computer vision towards embedding high-order constraints in MRFs. These includes, for instance, region connectivity \cite{globalinter,globalconn,graphcutbase}, shape convexity \cite{Gorelick2017}, curvature regularization \cite{Nieuwenhuis2014} and shape compactness \cite{Dolz2017}, among other high-order priors. In this paper, we investigate exact region connectedness priors. More precisely, we are interested in solving~\eqref{ILP} to global optimality, while adding a global (high-order) potential function to~\eqref{ILP} to explicitly enforce the connectivity of each label (to be made more precise in Sec.~\ref{connect}). 
A \emph{$k$-label partitioning} of the image in this paper is a partition of $G$ into connected subgraphs $\{G_1, G_2, \ldots, G_k\}$ such that
$\cup_{i=1}^kG_i=G$, and $G_i \cap G_j=\emptyset$, $i\neq j$. Without loss of generality, we assume that segment (subgraph) $G_i$ is assigned the label $i$.
Enforcing the connectivity potential itself is proven to be $\mathcal{NP}$-hard in~\cite{graphcutbase}.

%assume $k$ is fixed. We are not interested in general segments, but in segments which are in a certain sense connected.

%A \emph{(u,v)-path} is a sequence of adjacent edges in $E$ that starts at $u$ and ends at $v$ (without vertex repetitions). A \emph{cycle} is a path returning to the starting node. We call a vertex set~$S$ $\emph{connected}$ if for all pairs $u,v\in S$ there exists a $(u,v)$-path within~$S$. So in graph-theoretical terms the image segmentation problem consists of assigning (region) labels to the vertices such that each label represents a connected vertex set. 

\subsection{Related Works}
%While the computer vision community has made various progress on the supervised multi-label segmentation problem, the segmentation result does not have any guarantee on the number of connected regions.  In extreme cases, many separated tiny regions with the same labels are present.

Image segmentation under approximate connectivity constraints has been considered in~\cite{graphcutbase}, where a binary MRF is solved. Exact connectivity is not considered in~\cite{graphcutbase}. Instead, a simplified version of the problem is proposed, where only a given (user-provided) pair of nodes must be connected. Following this assumption, the problem is solved with a  heuristic-based graph cut algorithm \cite{experimental}, obtaining connected foreground (binary segmentation). 
%In this paper, we introduce the problem of \emph{Seeded Multi-Region Image Segmentation} (SMRIS), where the user can fix the number of connected regions (k) in advance, and draw k seeds on the original image. The seeds are also used for learning the intensity of each region. To the best of our knowledge, it is the first global model that is able to achieve this goal.

%LP-PC. 
Exact global connectivity potentials are formulated as an ILP in \cite{globalconn}, where connected subgraph polytopes are introduced. Due to the high computational cost of solving the corresponding $\mathcal{NP}$-hard problem, the work in \cite{globalconn} examined only LP relaxations of the ensuing ILP. Although the general formulation works for multi-label MRFs, the authors applied it only to binary MRF problems. %In the remainder of the paper, we denote this model as LP-PC, where LP denotes the LP relaxation, and P denotes the pairwise prior in \eqref{energy}. It is worth mentioning that the subgraph connectivity problem also plays an important role role in the  operations research community, and has been applied, for instance, to the forest planning problem \cite{imposeconn}, where each sub-region of the forest is constrained to be connected.
%ILP-C. 
  In \cite{minicost}, the authors optimized exactly a linear (unary-potential) objective subject to connectivity constraint in a binary segmentation problem. It solves two instances of medical benchmark datasets to optimality for the first time. However, the model does not apply to the general multi-label pairwise MRF objective in~\eqref{energy}, which is of wide interest in vision applications.
%It is claimed in \cite{minicost} that two medical benchmark datasets are solved to optimality for the first time. However, no pairwise priors are considered in \cite{minicost}. We extend this model to the multi-label case, and denote as ILP-C, where C stands for connectivity. 
Finally, it is worth mentioning that the subgraph connectivity problem also plays an important role in the  operations research community, and has been applied, for instance, to the forest planning problem \cite{imposeconn}, where each subregion of the forest is constrained to be connected.

%ILP-PC. 
%In this paper, we are interested in incorporate connectivity as hard constraint within each label of a multi-label MRF. We denote our main model as ILP-PC, where P denotes pairwise prior. A branch-and-cut method is applied to solve the resulting ILP to global optimality, where the connectivity priors are enforced iteratively by a cutting plane method.

\subsection{Contribution}

This paper investigates multi-label MRFs with exact connectivity constraints. To solve the ensuing ILP problem, we propose a branch-and-cut method, which provably finds globally optimal solutions. The method could provide feasible solutions with a guarantee on suboptimality even if we terminate it earlier. 
Unlike \cite{globalinter,globalconn}, which examines LP relaxations of the initial ILP, our method provides global optimality guarantee. Different from \cite{graphcutbase}, we consider exact connectivity and we do not reduce the problem to connectivity between a given pair of points. The proposed ILP is quite general, and can be applied as a post-processing method on top of any existing multi-label segmentation approach. As it provides globally optimal solution, it can be used off-line to generate ground-truth labeling, which serves as quality check for any fast on-line algorithm. Furthermore, it can be used to generate ground-truth proposals for state-of-the-art weakly supervised semantic segmentation techniques, e.g., those based on partial scribble-based annotations \cite{ScribbleSup}.  

%We demonstrate the power and usefulness of our model by several experiments in the BSDS500 image dataset, as well as medical images with trained probability maps. 

%The rest of the paper is organized as follows. In section 2, we derives the ILP formulation for 2D images, without explicit formulations of connectivity constraints. Section 3 introduces one mathematical programming model to enforce the connectivity within each region, i.e., the cutting plane approach. Section 4 discusses about strategies towards larger images, e.g., using the results of $\alpha$-expansion as the input. We conduct several computational experiments using medical images in Section 5, and concludes our paper in Section 6 with future work.

%\begin{itemize}
%\item We propose a branch-and-cut approach that solves the ILP of the multi-label MRF with connectivity priors to optimality.
%\item We propose an interactive framework so that no training data is needed. The unary data term of each label is learned from the seeds. In addition, the user can mark the outliers of one label as different labels.
%\item We design a fast region fusion based heuristic to provide a warm start for the ILP solver, which is often found to be near optimal in our experiments.
%\item We extend the multi-label MRF from piecewise constant images to the piecewise affine ones.
%\end{itemize}

\section{Connected Subgraph Polytopes}
\label{connect}
In this section, we introduce the convex hull of the set of all connected subgraphs, where a connected subgraph consists of nodes with the same label that are connected. We call a node $i\in V$ \emph{active} if $x_i=1$, e.g., if it is labeled as foreground.

\textbf{Connected Subgraph Polytope}.
Given a connected, undirected graph $G=(V,E)$, let $C=\{\x: G'(V',E') \;\text{connected}\}$, where $V' =\{i\in V: x_i=1\}$ and $E'=\{(i,j) \in E:i,j \in V'\}$. Recall that a subgraph $G'(V',E')$ is connected if $\forall i,j\in V'$, $\exists $ a path in $G'$ that connects $i$ and~$j$. Then, $C$ denotes the finite set of connected subgraphs of $G$, and we call the convex hull of $C$ the connected subgraph polytope of $G$, denoted by $conv(C)$.  It was proven in \cite{karp2002} that optimizing a linear
function over $conv(C)$ is $\mathcal{NP}$-hard.

%The $k$-label image segmentation problem we studied in this paper then becomes to find a suitable partition of $V$ into $\{C_1,C_2,\ldots, C_k\}$ such that $\cup_{i=1}^kC_i=V$, and $C_i \cap C_j=\emptyset$, $i\neq j$.

\textbf{Vertex-Separator Set}.
Given an undirected graph $G=(V,E)$, for any pair of active nodes $i,j \in V,\; i\neq j,\; (i,j)\notin E$, the set $S\subseteq V\setminus\{i,j\}$ is called a vertex-separator set with respect to $\{i,j\}$ if the removal of $S$ from $G$ disconnects $i$ and $j$ in $G$. 
As an additional definition, a set $\bar{S}$ is said to be a \emph{minimal vertex-separator set} if it is a  vertex-separator set with respect to a node pair $\{i,j\}$ in~$G$ while any strict subset $T\subset \bar{S}$ is not.

Let $\mathcal{S}(i,j)= \{S\subset V : S\; \text{is a vertex-separator with respect to}~\{i,j\}\}$ be the collection of all $\{i,j\}$ vertex-separator sets in $G$, and $\mathcal{\bar{S}}(i,j) \subset \mathcal{S}(i,j) $ be the subsets of minimal vertex-separator sets.

Following \cite{globalconn}, we can describe $C$ with the  class of linear inequalities
\begin{alignat}{2}
x_i + x_j -1 \leq \sum\optlimits_{s\in {S}} x_s,\;\;
\forall i,j\in V: (i,j)\notin E,\; \forall S\in \mathcal{S}(i,j), \label{seperator}
\end{alignat}
where $x_i\in\{0,1\}$, $i\in V$. Precisely, if two nodes $i$ and $j$ are active (left hand side of \eqref{seperator} becomes $1$), they are not allowed to be separated by any set of inactive nodes of $S$ (at least one node in any $S\in \mathcal{S}(i,j)$ must be active).

The convex hull of a finite set is the tightest possible convex relaxation, and facet-defining inequalities are true facets of the convex hull. In \cite{globalconn}, the authors prove inequalities~\eqref{seperator} are facet-defining for $conv(C)$ if $\mathcal{S}(i,j)$ is replaced by $\mathcal{\bar{S}}(i,j)$. However, the number of such constraints is exponential in $|V|$.

\textbf{Rooted case}.
%In many applications, it is reasonable to assume that a root node can be identified within each label, either manually or by a heuristic. 
In this paper, we require the user to input a scribble for each label, so that at least one node is identified within each label. Let $r$ denote the root node for the label (we use the first node of the scribble). Then, it suffices to check connectivity of every active node to the root node instead of all pairs of active nodes. Thus, constraints \eqref{seperator} become

\begin{equation}
x_i \leq \sum\optlimits_{s\in {S}} x_s, \;\;\forall  i\in V: (i,r)\notin E, \; \forall S\in \mathcal{S}(i,r). \label{root}
\end{equation}

\begin{theorem}
It is still $\mathcal{NP}$-hard to optimize over the connected subgraph polytope $\mathcal{C}$ even if one root node $r$ is given.
\end{theorem}

\begin{proof}
The proof can be found in the supplementary materials.
\end{proof}

In practice, the number of constraints \eqref{root} is still exponential in $|V|$ (as the number of vertex-separator set is), hence they cannot be considered all simultaneously for graphs of large sizes. 
However, given a labeling $\x$, we can identify a subset of violated connectivity constraints of type~\eqref{root} in polynomial time and iteratively add them to the ILP while searching 
for new integer solutions. This is known as the \emph{cut generation} approach. We will look into this in detail in Sec.~\ref{cutgen}.

%-----------------------------------------------------------------------
\section{MRFs with Connectivity Constraints}
\label{formula}
%Let $D\in \mathbb{R}^n$ denotes the image domain, and $f(x): D\rightarrow [0,1]$ denotes the observed image data.
%In the discrete setting 
%Let $f_i=f(i)$ be the intensity of pixel $i$.
%and $z_i = (z_i^x, z_i^y)$ be the coordinate of of node $i$ (if $i$ represents a superpixel, then $z_i$ is its central position). 
%We call an image piecewise constant if $f$ is piecewise constant. 
%and piecewise affine if $f$ is piecewise affine with respect to the coordinates $\z$.

%----------------------
\subsection{Proposed model: ILP-PC}

Let $f_i$ denotes the observed image feature (e.g., color) at spatial location $i$. 
%Given an approximately piecewise constant image, 
We assume the user inputs $k$ scribbles as seeds for the $k$ labels, as shown in the left image of Fig.~\ref{fish}. 
Assuming image observations follow a piecewise constant model within each region\footnote{We assume a piecewise constant model for simplicity. However, our formulation extends to any other probabilistic assumptions of observation models.}, let $Y_\ell$ denotes the image average of seeds within label $\ell$. In this case, unary potential $c_i^\ell=|f_i-Y_\ell|$ evaluates how well label $\ell$ fits node $i$.

Let $C_\ell$ denotes the connected subgraph polytope of label $\ell$, and $\x^\ell :=(x_1^\ell, \ldots, x_n^\ell)$. By introducing two nonnegative variables $\varepsilon_i^{\ell+}$ and $\varepsilon_i^{\ell-}$ to model $|x_i^\ell-x_j^\ell|$, the ILP of our multi-label MRF with connectivity constraints becomes:
\begin{alignat}{2}
\min_x\;(1-\lambda)&\sum_{\ell=1}^k\sum_{i=1}^n c_i^\ell x_{i}^\ell +\lambda\sum_{\ell=1}^k\sum_{(i,j)\in E}(\varepsilon_i^{\ell+} +\varepsilon_i^{\ell-})\label{ILP1d}\\
\sum\optlimits_{\ell=1}^kx_{i}^\ell&=1, \;\;\forall i\in [n],\; \ell\in [k], \tag{\ref{ILP1d}a} \\
x_i^\ell-x_j^\ell &= \varepsilon_i^{\ell+} -\varepsilon_i^{\ell-},\;\;\forall i\in [n],\; \ell\in [k], \tag{\ref{ILP1d}b}\label{ILP1db}\\
x_{i}^\ell&\in\{0,1\}, \;\;\forall i\in [n], \; \ell\in [k], \tag{\ref{ILP1d}c}\label{ILP1dseed}\\
\x^\ell &\in C_\ell,\;\;\forall \ell\in [k],\tag{\ref{ILP1d}d}\label{ILPconn}\\
\varepsilon_i^{\ell+}, \varepsilon_i^{\ell-} &\geq 0,\;\;\forall i\in [n], \; \ell\in [k],\tag{\ref{ILP1d}e}\\
x_{i}^\ell&=1, \;\;\forall i \text{ within the scribble of label $\ell$}, \tag{\ref{ILP1d}f}\label{ILP1dseed2}
\end{alignat}
where constraints \eqref{ILPconn} can be expressed as the rooted vertex-separator constraints~\eqref{root}. The transformation holds because~\eqref{ILP1d} is a minimization problem, and $|x_i^\ell-x_j^\ell| = 1$ will only induce $\varepsilon_i^{\ell+} = 1$ and $\varepsilon_i^{\ell-} = 0$.

In the case of a superpixel graph, where a superpixel contains similar pixels in terms of color or texture, we represent relations between neighboring superpixels by defining the corresponding \emph{Region Adjacency Graph} (RAG) $G=(V,E)$, where $E$ contains edges between pairs of adjacent superpixels. 
We multiply the unary data term by $\sigma_i$ and the pairwise term by $\gamma_{ij}$. Here, $\sigma_i$ denotes the number of pixels contained in node (superpixel) $i$, and $\gamma_{ij}$ represents the number of neighboring pixels between node $i$ and $j$.

%----------------------
\subsection{ILP-PCB: ILP-PC with background label}
If a clear background (not necessarily connected) exists in the given image, the connectivity constraints can be ignored on the specific label, which we call the background label. 
This is a reasonable assumption in many cases, such as the black-region background in the left image of Fig.~\ref{swan}. In this example, the background has $4$ disconnected components. 
Fig.~\ref{swan} depicts results with and without background label.
%ILP-PCB in the bottom row of Fig.~\ref{swan} outperforms ILP-PC in the second row, in terms of accuracy and speed.

\begin{figure}
\centering
\includegraphics[width=0.5\columnwidth]{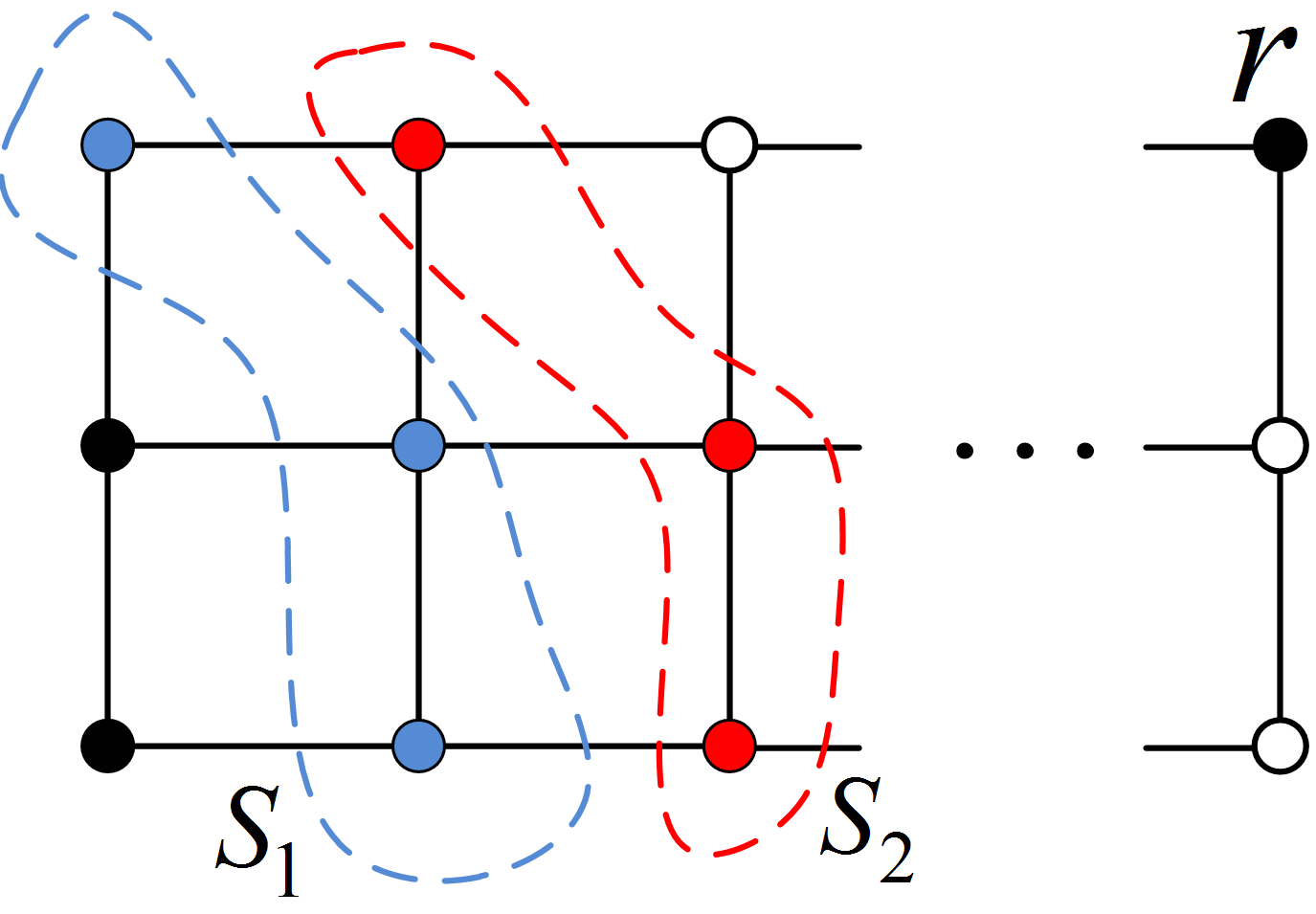}
\caption{\textit{$K$-Nearest cut generation strategy}. Active nodes are shown in black, and the two separator sets are marked in red and blue.}
\label{Cutting plane}
\end{figure}
%-----------------------------------------------------------------------
\section{Solution Techniques}
\subsection{Towards global optima: the branch-and-cut method}
\label{cutgen}
The most widely used exact method for solving an ILP is \emph{branch and cut}, namely a sophisticated combination of the \emph{branch and bound} and the \emph{cutting plane} algorithms \cite{Lodi2010}. In this section, we focus on the cutting plane method because the branch-and-bound method is implemented by default in any modern ILP solver. We recall that the ILP gap of a minimization ILP problem is computed as $(I-LP) / I$, where $I$ denotes the best integer solution and $LP$ the LP-relaxation value. The ILP solver terminates if the ILP gap is very small, or if the time limit is reached.

We are interested in exact connectivity and we focus on the rooted case \eqref{root}.
%since we assume at least one root node~$r$ is fixed for each label by the user interaction. 
We concentrate on enforcing the connectivity constraints for one label only (e.g., $\ell$). Then, the same approach will be repeated for other labels until they are all connected (in the case of a background label, we ignore its connectivity). 

The basic idea is to omit~\eqref{ILPconn} initially, explore the branch-and-bound tree of system~\eqref{ILP1d} until an integer solution is found and then check the feasibility of this solution (i.e., connectivity). If infeasible, violated constraints~\eqref{ILPconn} are identified (by solving the so-called the \emph{separation problem}) and added to~\eqref{ILP1d} to cut off the infeasible solution. This procedure is iterated until~$G_\ell$ is connected.

We treat individual connected components (see Fig.~\ref{Cutting plane}) as one entity, since establishing connectivity between all nodes in this component and $r$ automatically connects all the nodes. Then, the separation problem boils down to finding a vertex-separator set $S$ between each disconnected active component and the root component (containing $r$) in the current solution.

At the heart of the above technique is that only a subset of connectivity constraints~\eqref{ILPconn} will be active at the optimum of~\eqref{ILP1d}, i.e., polynomially many should be enough for the algorithm to converge to the optimal solution. However, depending on the choice of the inequalities at each step, we may require a different number of such cutting planes and the number of iterations varies. 

\textbf{The separation problem and cutting plane selection}.
Among the many ways of separating and selecting the violated constraints \eqref{ILPconn}, we choose the so-called $K$-Nearest strategy. Note that, although the separation problem for constraints \eqref{ILPconn} is  generally $\mathcal{NP}$-hard (see Sec. \ref{connect}), it is polynomial to separate integer infeasible solutions, for example, by means of the following algorithm.  
%We list two strategies among the many proposed in~\cite{globalconn, minicost}.
%
%\textbf{Minimal Separator.} A minimal (in terms of $|S|$) separator set is obtained by solving a max-flow problem between any two disjoint active components %(or towards the rooted component) and selecting the smaller vertex set on either side of the resulting min-cut. The strategy was applied in~\cite{globalconn}.
%
Precisely, we run a breath-first search for any active component $H$ to collect the $K$ (disjoint) vertex separator sets~$S_m$ ($m=1, \dots, K$) composed of all nodes with identical distance. 
The search terminates if $K$ equals the number of nodes in $H$ or if another active node is reached. The idea is illustrated in Fig.~\ref{Cutting plane}, where active nodes are shown in black and $r$ denotes the root node. The two separator sets are marked in red and blue. Here $K=2$, because it reaches the number of nodes in~$H$. 
%We iterate the above procedure of adding cuts until every $G_\ell$ is connected. 

The $K$-Nearest strategy is reported in~\cite{minicost} to be one of the most successful (among five) in terms of solved instances and computational efficiency. We will adopt this vertex separation strategy in Sec.~\ref{experiments}. 

\subsection{$L_0$-$H$: a region fusion based heuristic}
To improve the efficiency of the ILP solver while solving problem \eqref{ILP1d} with the method described in the previous section, we calculate an initial feasible solution with a heuristic, called $L_0$-$H$, which will be used as an upper bound of 
the branch-and-bound tree. On the one hand, this helps to prune a lot of unnecessary branching nodes. On the other hand, the solver can provide an optimality gap to the initial solution, by solving the LP relaxation of the ILP, which will serve as a lower bound to the problem.

We adopted the idea for the heuristic from~\cite{Fast}, which is basically a local greedy algorithm to solve the discrete Potts model~\cite{potts}. It works by 
iteratively merging groups of nodes. 

In the beginning, each scribble of 
nodes (superpixels) and every node not covered by any scribbles are in their own groups. Then 
for every iteration,
we merge two neighboring groups, if the following condition holds and the merging 
does not result in the nodes of two different scribbles being in the same group: 
\begin{equation}
\label{lzero}
 \sigma_i\cdot \sigma_j\cdot|Y_i-Y_j| \leq \beta\cdot \gamma_{ij}\cdot(\sigma_i+\sigma_j).
\end{equation}
where $\sigma_i$ denotes the number of pixels in group (segments) $i$, $\gamma_{i_j}$ denotes the number of neighboring pixels (boundary length) of two groups $i$ and $j$, and $Y_i$ the mean of image data (e.g., color) within group $i$.
By increasing parameter~$\beta$ in \eqref{lzero} in every iteration, we terminate the algorithm when exactly $k$ 
groups remain.
It is shown in~\cite{Fast} that the following exponentially growing strategy of $ \beta $ gives the best results.
$$ \beta = (\frac{\text{iter}}{100})^{2.2}*\eta$$
where iter is the current iteration number, and $\eta$ is the regularization parameter for the Potts model.
We will show in Sec.~\ref{experiments} that $L_0$-$H$ is fast and generates good results most of the time, sometimes even optimal.

\section{Experiments}
\label{experiments}
In this paper, all computational experiments were performed using Cplex $12.7.0$, on a Intel i5-4570 quad-core machine, with $16$ GB RAM. 
We show experiments on medical images, where the unary potentials are based on the probability maps of given labels, which were trained using 
convolutional neural networks (CNN) \cite{Dolz2017a}. The sizes range from $96\times96$ to $256\times256$. 

We further use the Berkeley Segmentation Dataset~\cite{amfm_pami2011} (BSDS500, image size $321\times 481$) and the PASCAL VOC 2012 set~\cite{parscal2012} (PASCAL, image size around $500\times 400$). 
We first apply the SLIC~\cite{slic} superpixel algorithm to get an over-segmentation, with the number of superpixel around~$1000$. 
%We then use the gray-scaled intensity of the superpixel image.

Using superpixels has several advantages. First, the complexity of the optimization problem is drastically reduced with only a negligible segmentation error. Second, the information in each superpixel is more discriminative, and also overcomes the case of outliers. 
As shown in a recent superpixel algorithms survey paper~\cite{stutz}, a few advanced superpixel algorithms can achieve very accurate over-segmentation results with around $1000$ suerpixels.

%\subsection{The models to be compared}

We conduct a comprehensive comparison of the following different optimization models:

\begin{itemize}
%\item ILP-C. Formulation~\eqref{ILP1d} without pairwise data term, which was applied in~\cite{minicost} for a binary MRF with exact connectivity priors~\eqref{ILPconn}.

\item ILP-PC. Our proposed ILP formulation~\eqref{ILP1d} of multi-label MRF, under the global connectivity constraints.

\item LP-PC: The LP relaxation of ILP-PC, which was introduced in ~\cite{globalconn,minicost}.

\item $L_0$-$H$: Our proposed $L_0$ region fusion based heuristic, which was motivated by ~\cite{Fast} and modified to generate exactly $k$ connected segments.

\item ILP-P: The ILP formulation of~\eqref{ILP1d} without connectivity constraints~\eqref{ILPconn}, which is widely used in vision (e.g., graph cuts).

\item ILP-PCB: ILP-PC with the ``background'' label marked by the user, where this special label is not required to be connected.

\end{itemize}

In this section, if there is no further explanation, the default setting for the pairwise potential $\lambda$ is $0.2$, $100$~secs for the time limit, and $0.1$ for the $L_0$-$H$ parameter~$\eta$. We adopt $L_0$-$H$ to provide initial solution for the ILP solver. When we report energy $E$, we use the objective function in \eqref{ILP1d}.
%, except for ILP-C, where no pairwise prior energy is counted. (Hence often results in lower energy)

\begin{figure}[t!]
    \centering
    \begin{subfigure}{0.23\columnwidth}
                    \includegraphics[width=.99\columnwidth]{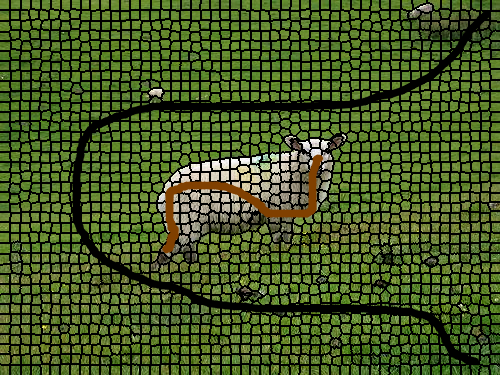}
    \end{subfigure}
    \begin{subfigure}{0.23\columnwidth}
            \includegraphics[width=.99\columnwidth]{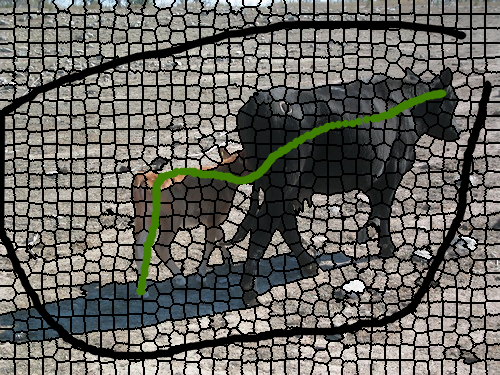}     
    \end{subfigure}       
    \begin{subfigure}{0.23\columnwidth}
            \includegraphics[width=.99\columnwidth]{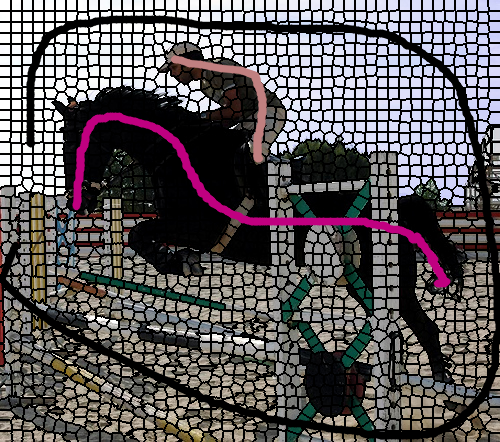}
    \end{subfigure}    \\    
    \begin{subfigure}{0.23\columnwidth}
            \includegraphics[width=.99\columnwidth]{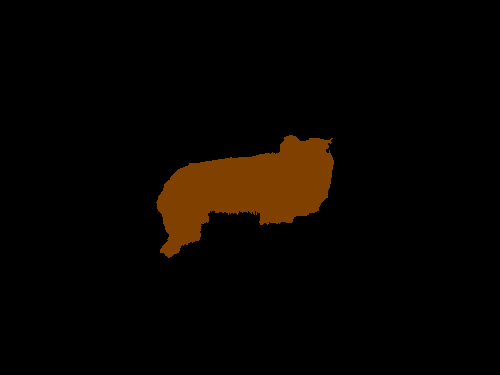}
    \end{subfigure}
    \begin{subfigure}{0.23\columnwidth}
            \includegraphics[width=.99\columnwidth]{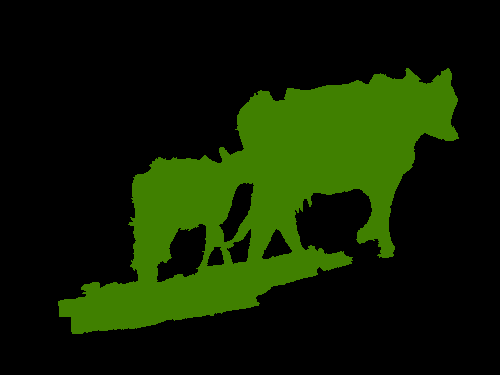}
    \end{subfigure}        
    \begin{subfigure}{0.23\columnwidth}
            \includegraphics[width=.99\columnwidth]{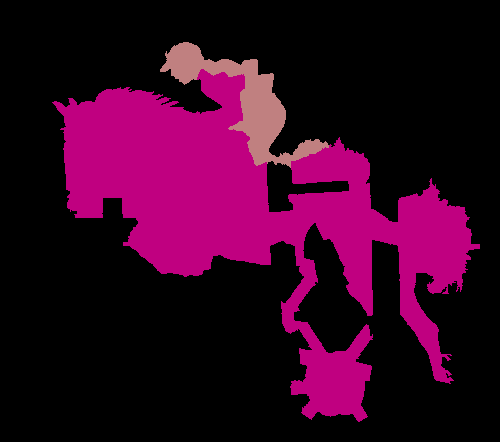}
    \end{subfigure}
     \caption{Ground-truth generation on $3$ images taken from PASCAL. $\lambda = 0.3$ for the third image, $t=0.08$, $0.15$ and $31.3$ secs, $E=5927.4$, $12220.1$ and $28238.5$.}
     \label{pascal_test}
\end{figure}

\subsection{Ground-truth generation}
Our proposed ILP solver is $\mathcal{NP}$-hard, but provides global optimal solution for the multi-label MRFs with connectivity prior. Thus, it could be used off-line to generate ground-truth labeling, which serves as quality check for any fast on-line algorithms.

We conduct experiments on two instances taken from PASCAL, where the scribbles of all $11k$ training images are online available and provided by ScribbleSup~\cite{ScribbleSup}. We set $\lambda$ equals $0.2$ for the first two image and $0.3$ for the third. ILP-PC takes only $0.08$ and $0.15$ secs on the first two instances and $31.3$ secs on the third to get the optimal solution. The optimal energy are $5927.4$, $12220.1$ and $28238.5$ respectively.

\subsection{Detailed comparison of different models}
\subsubsection{Medical images with probability maps}
\label{medical}

\begin{figure}[t!]
    \centering
    \begin{subfigure}{0.23\columnwidth}
                    \includegraphics[width=.99\columnwidth]{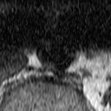}
                	\caption{Medical image}
    \end{subfigure}%
    \begin{subfigure}{0.23\columnwidth}
            \includegraphics[width=.99\columnwidth]{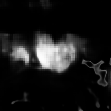}
            \caption{Probability map }      
    \end{subfigure}  
    \begin{subfigure}{0.23\columnwidth}
            \includegraphics[width=.99\columnwidth]{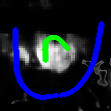}
            \caption{User scribbles}      
    \end{subfigure}\\
    \begin{subfigure}{0.23\columnwidth}
            \includegraphics[width=.99\columnwidth]{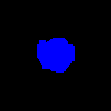}
            \caption{$L_0$-$H$, 0.83s. }      
    \end{subfigure}          
     \begin{subfigure}{0.23\columnwidth}
                    \includegraphics[width=.99\columnwidth]{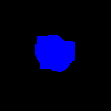}
                \caption{ILP-PC, 100s.}
    \end{subfigure}%
    \begin{subfigure}{0.23\columnwidth}
            \includegraphics[width=.99\columnwidth]{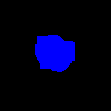}
            \caption{ILP-PCB, 100s. }      
    \end{subfigure}   \\   
     \begin{subfigure}{0.23\columnwidth}
                    \includegraphics[width=.99\columnwidth]{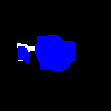}
                \caption{LP-PC, 1.19s. }
    \end{subfigure}%
    \begin{subfigure}{0.23\columnwidth}
            \includegraphics[width=.99\columnwidth]{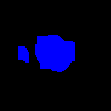}
            \caption{ILP-P, 0.46s. }      
    \end{subfigure}
    \begin{subfigure}{0.23\columnwidth}
            \includegraphics[width=.99\columnwidth]{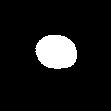}
            \caption{Ground truth. }      
    \end{subfigure}
     \caption{Comparison of 5 models on a medical image, where user scribbles are depicted on the probability map. Energy is reported in Sec.~\ref{medical}, and numbers denote the time spent. In LP-PC, $0.62\%$ of the pixels remains unlabeled , colored in white. Both ILP-PC and ILP-PCB have $2.8\%$ ILP gap.}
     \label{Aorta_MRI}
\end{figure}

\begin{table}[h]
\centering
\begin{tabu} { | X[c] | X[c] | X[c]   | X[c] | X[c] |}
 \hline
 $L_0$-$H$ & ILP-PC & ILP-PCB  & LP-PC & ILP-P \\
 \hline
 864.5  & 854 & 854 & 829.5 & 826.8 \\
\hline
\end{tabu}
\caption{Energies of $5$ proposed models on an MRI image.}
\label{table:energy}
\end{table}
We report a medical image segmentation example, where unary potentials are based on the probability maps of given labels, which were trained using convolutional neural networks (CNN) \cite{Dolz2017a}. The purpose here is to obtain a binary (two-region) segmentation of a magnetic resonance image (MRI), which depicts the abdominal aorta~\cite{Dolz2017}. In this example, the CNN probability maps yielded unsatisfying disconnected region due to imaging noise, the lack of boundary contrast and limited training information.

The input image is of size $111\times 111$, and the computation time is reported in Fig.~\ref{Aorta_MRI}, where ILP-PC and ILP-PCB both failed to converge. $L_0$-$H$ result is of high quality, within $1.22\%$ of the best solution found by ILP-PC in $100$ seconds.
%, and equals that of ILP-C upon time limit. 
The energy of all models is reported in table~\ref{table:energy}.

As we see in Fig.~\ref{Aorta_MRI}, LP-PC has $0,62\%$ fractional solution (depicted in white). Although a post-processing rounding heuristic can be applied, it is not guaranteed even to find a feasible solution. ILP-P gives two separated regions, which is far away from the ground truth.
We notice that LP-PC and ILP-P give lower energy than ILP-PC. This is because both of them are relaxations for ILP-PC and, therefore, provide lower bounds. The inclusion of background label is not beneficial in this example.

\subsubsection{Superpixels of BSDS500}
%As is often the case, it is reasonable to assume the $k$-th label is the background of an image. In such cases, connectivity constraints are needed for the first $k-1$ label sets. Figure ~1 is such case.
In this section, we introduce another model ILP-PCW, which is ILP-PC without the initial solution of $L_0$-$H$. The purpose is to test whether the ILP solver is able to achieve good results by itself.
\begin{figure}[t!]
    \centering
    \begin{subfigure}{0.36\columnwidth}
                    \includegraphics[width=.99\columnwidth]{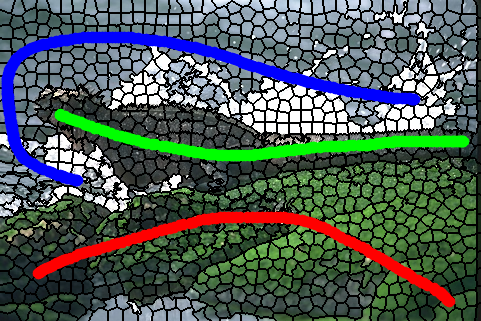}
                \caption{Input with $3$ user scribbles.}
    \end{subfigure}
    \begin{subfigure}{0.36\columnwidth}
            \includegraphics[width=.99\columnwidth]{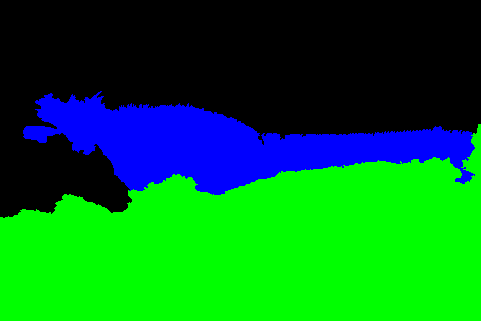}
            \caption{$L_0$-$H$, 0.04s, E=16088.}      
    \end{subfigure}\\        
    \begin{subfigure}{0.36\columnwidth}
            \includegraphics[width=.99\columnwidth]{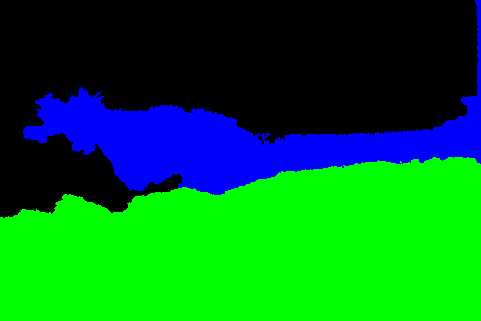}
            \caption{ILP-PC, 100s, E=15804.9.}
    \end{subfigure}        
    \begin{subfigure}{0.36\columnwidth}
            \includegraphics[width=.99\columnwidth]{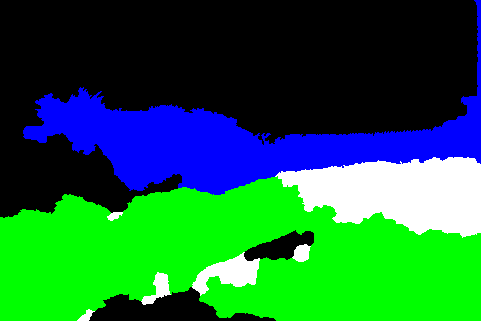}
            \caption{LP-PC, 0.27s, E=15560.9.}
    \end{subfigure}\\
    \begin{subfigure}{0.36\columnwidth}%
            \includegraphics[width=.99\columnwidth]{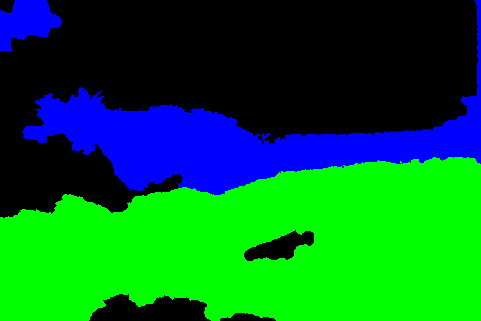}
            \caption{ILP-P, 0.08s, E=15232.5.}
    \end{subfigure}%
    \begin{subfigure}{0.36\columnwidth}
            \includegraphics[width=.99\columnwidth]{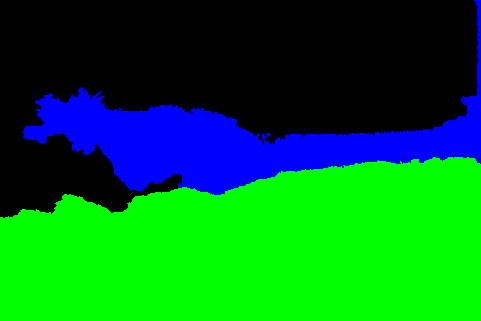}
            \caption{ILP-PCW, 61s, E=15804.9.}
    \end{subfigure}
     \caption{Comparison of $5$ models on BSDS, including ILP-PCW ( ILP-PC without initial solution from $L_0$-$H$). Note that $5.9\%$ of the nodes remains unlabeled in LP-PC, colored in white. The $L_0$-$H$ solution is within $1.76\%$ of best solution.}
     \label{fish}
\end{figure}

Fig.~\ref{fish} depicts an example, where ILP-PC with $L_0$-$H$ does not converge within the time limit while ILP-PCW finds the provably global optimal solution. 
The solution time and the energy are reported in the figure. The energy of the starting solution provided by $L_0$-$H$ is very good, within $1.76\%$ of the optimal 
solution found by ILP-PCW in $61$ seconds. Note that ILP-PC and ILP-PCW give the same energy, meaning they found the same solution, but ILP-PC failed to get the tightest
lower bound, having an ILP gap of $0.3\%$.
A closer look into the log file of Cplex shows that, given the good initial solution of $L_0$-$H$, ILP-PC found this best solution in less than $1$ sec, while ILP-PCW 
takes $18$ secs.

The inclusion of the integrality constraints~\eqref{ILP1dseed} and the connectivity priors greatly improve solution quality. As many as $5.9\%$ superpixel values of LP-PC are factional. In ILP-P, the green and black labels have several disconnected regions, resulting in a worse solution.

%In the last row of Fig.~\ref{fish}, we show ILP-C and ILP-PC without $L_0$-$H$ as initialization. This shows that, while the ILP solver can benefit from having a good initial solution, it can also generate very good results.

\subsubsection{More experiments on BSDS500 images}
More experiments on BSDS500 images are shown in Fig.~\ref{large}. In the first column, the pairwise term $\lambda$ is set to $0.1$ to encourage thin branches of the tree, while all other parameters remain at their default values ($0.2$). We draw much fewer brushes in the right two columns, to show the robustness of our model (to be discussed in Sec.~\ref{brush} with more details). We observe that $L_0$-$H$ gives good results in the right two images, while not being satisfying in the left two cases. Our proposed model ILP-PC achieves the best overall results.
%With pairwise priors in ILP-PC, the results are more compact than ILP-C, as shown in the right two columns.

\subsection{Quantitative comparison of different models}
In this section, we give a detailed analysis of the $5$ different models with respect to energy, computational time and parameters. They are based on computational experiments of $15$ images from BSDS500 and medical images from~\cite{Dolz2017}.

\subsubsection{Statistics of $5$ models.}
We report the average running time of all models in the second row of Table~\ref{table:ILP_time}, where the time limit is $100$ sec. The average ILP optimality gap is shown in the third row, where ``Null'' means no ILP gap exists (because they are not an ILP). 
We can see that the ILP-PC on average takes $62.3$ seconds, and the ILP gap is $3.7\%$. Moreover, the inclusion of the ``background label'' (ILP-PCB) helps in term of both speed and ILP gap.  There exist two instances where ILP-PCB reduces the time of $100$ (time limit reached) secs from ILP-PC to less than one sec. This gain results from ``relaxing'' one label to be non-connected.
It is also surprising to see that ILP-P with only pairwise prior is also very efficient.

\begin{table}[t!]
\centering
\begin{tabu} {| X[c]| X[c] | X[c] | X[c] | X[c] | X[c] |}
 \hline
 &$L_0$-$H$   & ILP-PC &ILP-PCB        & LP-PC  &ILP-P  \\
 \hline
Time &0.7  & 62.3 & 39.2  & 1.4 & 0.3  \\
 \hline
Gap &Null  & 3.7\%   & 1.9\% &  Null  & 0  \\
\hline
\end{tabu}
\caption{Time and optimality gap of 5 proposed models.}
\label{table:ILP_time}
\end{table}

Apart from the above statistics, we also report that among all tested images an average of $3.5\%$ pixels found by LP-PC remain unlabeled (fractional solutions). Hence, we argue that LP-PC is not applicable in practice.

%-------------------
\subsubsection{ILP-PC against $L_0$-$H$}
In the conducted $15$ experiments, $L_0$-$H$ is found to be fast and it provides an initial solution to the ILP solver. ILP-PC adopts this initial solution, and seeks for better ones as the branch-and-cut tree proceeds. On average, the ILP-PC is able to improve $6.4\%$ quality of the initial solution (provided by $L_0$-$H$) within the time limit. Moreover, it could provide any feasible solution a lower bound (thus a sub-optimality guarantee) upon solving the LP-relaxation of the ILP.

%-------------------
\subsubsection{ILP-PC with $2$ secs time limit}
We further conduct experiments on the same $15$ instances with a time limit of ILP-PC set to $2$ secs. $L_0$-$H$ is again applied as pre-processing for the solver. We observe that ILP could improved $12$ out of the $15$ instances, and increase on average $4.4\%$ quality of initial solution in just $2$ secs. 
Hence, we argue that our proposed ILP model can be beneficial even within very short time, and thus applicable in much wider scenarios.

\begin{figure}[t!]
    \centering
    \begin{subfigure}{0.35\columnwidth}
                    \includegraphics[width=.99\columnwidth]{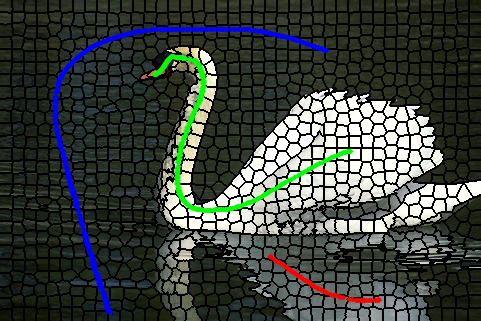}
                %\caption{x}
    \end{subfigure}%
    \begin{subfigure}{0.35\columnwidth}
            \includegraphics[width=.99\columnwidth]{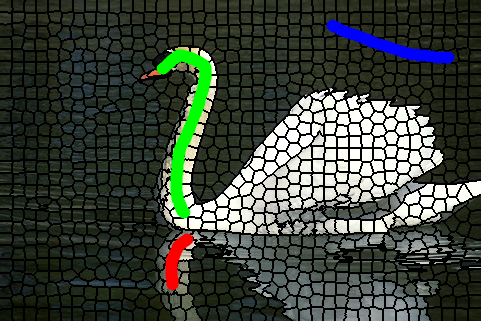}
            %\caption{x}      \label{sf:Second_Figure}
    \end{subfigure}\\        
    \begin{subfigure}{0.35\columnwidth}%
            \includegraphics[width=.99\columnwidth]{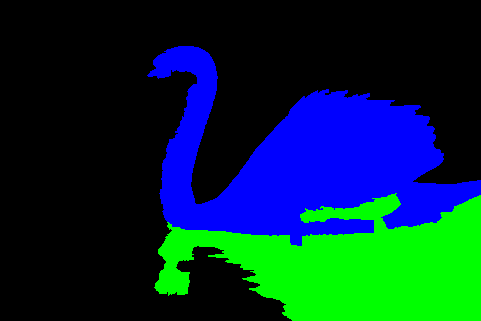}
            %\caption{x}
    \end{subfigure}%
    \begin{subfigure}{0.35\columnwidth}
            \includegraphics[width=.99\columnwidth]{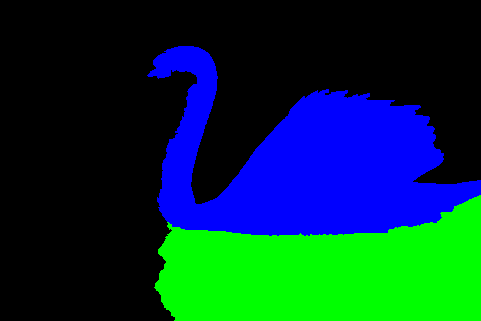}
    \end{subfigure}\\
    \begin{subfigure}{0.35\columnwidth}%
            \includegraphics[width=.99\columnwidth]{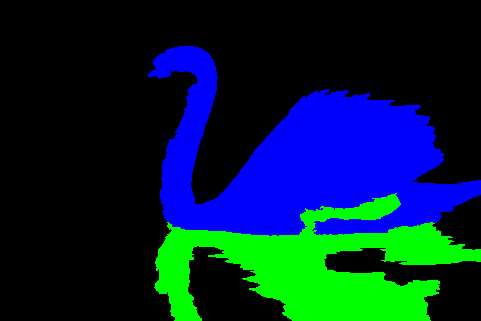}
            %\caption{x}
    \end{subfigure}%
    \begin{subfigure}{0.35\columnwidth}
            \includegraphics[width=.99\columnwidth]{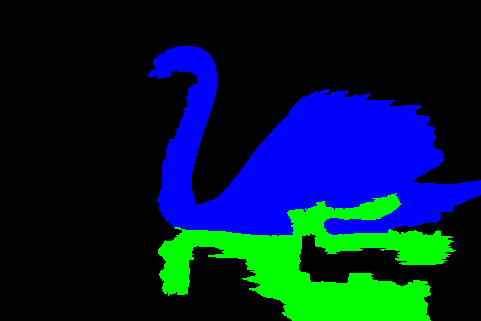}
    \end{subfigure}\\
     \caption{ Different user scribbles on the same image. Second row: ILP-PC with E=9544.2 and 13422.8, both reaching 100 seconds time limit. Bottom row: ILP-PCB with E=8615.3 and 10482.3, t= 2.1s and 0.4s. Note that the background label (shown in black) can be disconnected, while the other labels (blue and green) are connected.}
     \label{swan}
\end{figure}

\subsection{Analysis of different user scribbles}
\label{brush}
The user scribbles are used to learn the average color of each label, which is used in the ILP as the unary term. They also enforce hard constraints~\eqref{ILP1dseed2} into the ILP~\eqref{ILP1d}, which help fixing some of the binary variables, thus pruning the branch-and-bound search trees within the ILP solver. Moreover, in case of difficult situations, scribbles can also be used to exclude outliers from one label, such as in Fig.\ref{Aorta_MRI}. We show in Fig.~\ref{swan} that changing the scribbles does not alter significantly the results. While ILP-PC reaches the time limit in both cases, ILP-PCB gets the reported optimal solution in only $2.1$ and $0.4$ secs. The energy differences between the two cases are due to two factors: different scribbles resulted not only into different hard constraints, but also different unary potentials.

\begin{figure*}[t!]
 \center
 \begin{subfigure}{0.99\textwidth}
            \includegraphics[width=.99\columnwidth]{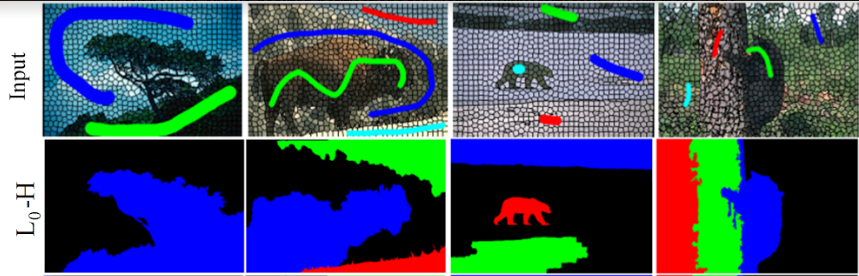}
            %\caption{x}      \label{sf:Second_Figure}
    \end{subfigure}\\
    \begin{subfigure}{0.99\textwidth}
            \includegraphics[width=.99\columnwidth]{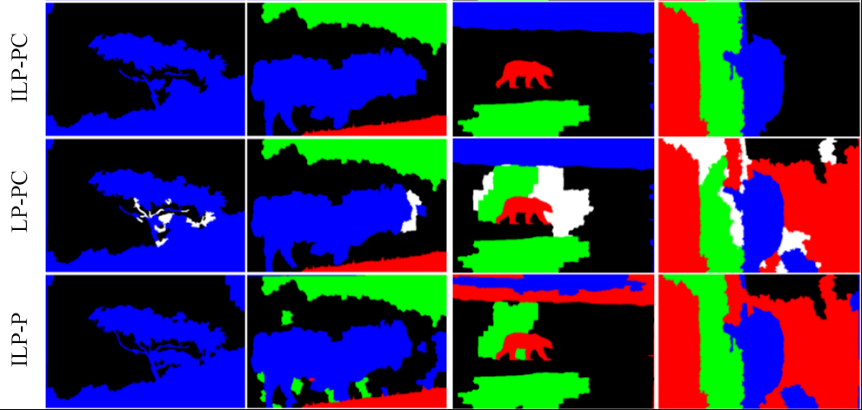}
            %\caption{x}      \label{sf:Second_Figure}
    \end{subfigure}
 \caption{More experiments on BSDS500 images. The pairwise term $\lambda$ is set to $0.1$ to encourage thin branches of the tree in the first column. 
 The user scribbles in the right two columns are fewer, to show the robustness of our model. 
 The white pixels in LP-PC denote fractional solutions, and ILP-P is without connectivity constraints, thus allowing disconnected regions with the same label. 
 %We observe $L_0$-$H$ gives good results for the two images on the right side, while the two cases on the left side are not satisfying. Our proposed model ILP-PC achieves the best overall results.
 }
 \label{large}
\end{figure*}

\section{Conclusion}
Recent years' algorithmic advances in Integer Programming plus the hardware improvements have resulted in a enormous speedup in solving ILPs. 
We revisit the ILP of the multi-label MRF with connectedness priors, and propose an exact branch-and-cut approach that enforce the connectivity constraints on the fly through cutting plane generation. A fast region fusion based heuristic is designed to provide a good initial solution.
The solver provides a nearly-optimal solution with a guarantee on the sub-optimality even if we terminate it earlier.

The ILP model can be applied to generate ground-truth labeling and segmentation offline, thus providing a quality assessment for any fast algorithm.

It can also be applied as a post-processing method on top of any existing multi-label segmentation approach. Hence, the advantage of ILP is two-fold. On the one hand, it provides a guarantee (lower bound) for any given initial solution. On the other hand, it seeks for better solutions during its search in the branch-and-bound tree, and it is beneficial even within very short time.

In this paper, we demonstrated the power and usefulness of our model by some experiments on the PASCAL, BSDS500 dataset, and medical images with trained probability maps. We have shown that with moderate-sized images or superpixels of large ones, our model achieves the best overall performance, yielding a provably global optimum in some instances.

\bibliographystyle{splncs}
\bibliography{research}

\section{Appendix}
\subsection{Proof for Theorem $1$}

\begin{proof}
The problem $C0$ of enforcing connectivity on one label in~\cite{graphcutbase}  is proved to be  $\mathcal{NP}$-hard. Suppose by fixing a root node, this problem becomes polynomial solvable. Then, we can randomly assign a node $v\in V$ to be the root node, and solve the resulting problem in polynomial time. Since there are $n$ ($n=\lvert V\rvert$) possible root nodes, by trying out all $n$ possible root nodes (still in polynomial time), we are sure to find the optimal solution out of $n$ optimization problems. Thus a contradiction.
\end{proof}

\end{document}